\documentclass{amsart}

\usepackage{graphicx} 
\usepackage[all]{xy} 

\usepackage{amssymb}
\usepackage{amsmath}
\usepackage{latexsym}
\usepackage[mathscr]{euscript}
\usepackage{pb-diagram}
\usepackage{latexsym}
\usepackage{epsfig}
\usepackage{graphics}
\usepackage{mathabx}
\usepackage{fullpage}

\newcommand\Rr{\mathbb{R}}
\newcommand\R{\mathbb{R}}

\newcommand{\beq}[1]{\begin{equation}\label{#1}}
\newcommand{\eeq}{\end{equation}}
\newcommand{\dgro}[2]{d_{\mathcal G\mathcal H}({#1},{#2})}

\def\compcirc {\mbox{\hspace{.05cm}}\raisebox{.04cm}{\tiny  {$\circ$ }}}

\newcommand\commentout[1]{\marginpar{\tiny $\backslash$commentout}}

\newcommand\dmtwo[1]{\left(\begin{smallmatrix}
			0 & {#1}\\
     			{#1} & 0
                     \end{smallmatrix}\right)}

\newtheorem{Theorem}{Theorem}[section]

\newtheorem{Definition}{Definition}[section]
\newtheorem{Remark}{Remark}[section]
\newtheorem{Example}{Example}[section]
\newtheorem{Proposition}{Proposition}[section]

\title{Persistent Clustering and a Theorem of J. Kleinberg}

\author{Gunnar Carlsson   and   Facundo M\'emoli}
\thanks{This work is supported by DARPA grant HR0011-05-1-0007.} \address{Department of Mathematics, Stanford University,
Stanford, CA 94305, USA.}
\email{$\{$gunnar,memoli$\}$@math.stanford.edu}
\date{August 15, 2008}

\keywords{Clustering, hierarchical clustering, persistent topology, categories, functoriality,Gromov-Hausdorff distance.}

\subjclass[2000]{Primary 62H30; Secondary 91C20}
\urladdr{http://comptop.stanford.edu/ }

\begin{document}

\begin{abstract} 
We construct a framework for studying clustering algorithms, which
includes two key ideas: {\em persistence} and {\em functoriality}.
The first encodes the idea that the output of a clustering scheme
should carry a multiresolution structure, the second the idea that one
should be able to compare the results of clustering algorithms as one
varies the data set, for example by adding points or by applying
functions to it.  We show that within this framework, one can prove a
theorem analogous to one of J. Kleinberg \cite{kleinberg}, in which
one obtains an existence and uniqueness theorem instead of a
non-existence result.  We explore further properties of this unique
scheme, stability and convergence are established.
\end{abstract} 

\maketitle


\section{Introduction}
Clustering techniques play a very central role in various parts of
data analysis.  They can give important clues to the structure of data
sets, and therefore suggest results and hypotheses in the underlying
science.  There are many interesting methods of clustering available,
which have been applied to good effect in dealing with many datasets
of interest, and they are regarded as important methods in exploratory
data analysis.

Despite being one of the most commonly used tools for unsupervised
exploratory data analisys and despite its and extensive literature
very little is known about the theoretical foundations of clustering
methods.

The general question of which methods are ``best", or
most appropriate for a particular problem, or how significant a
particular clustering is has not been addressed as frequently.  One
problem is that many methods involve particular choices to be made at
the outset, for example how many clusters there should be, or the
value of a particular thresholding quantity.  In addition, some
methods depend on artifacts in the data, such as the particular order
in which the elements are listed.  In \cite{kleinberg}, J. Kleinberg
proves a very interesting impossibility result for the problem of even
defining a clustering scheme with some rather mild invariance
properties.  He also points out that his results shed light on the
trade-offs one has to make in choosing clustering algorithms.  In this
paper, we produce a variation on this theme, which we believe also has
implications for how one thinks about and applies clustering algorithms.

In addition, we study the precise quantitative (or metric) stability
and convergence/consitency of one particular clustering scheme which
is characterized by one of our results. 

We summarize the two main points in our approach.

\medskip
\noindent {{\bf Persistence:} We believe that the output of clustering algorithms
shouldn't be a single set of clusters, but rather a more structured
object which encodes ``multiscale" or ``multiresolution" information
about the underlying dataset.  The reason is that data can often
intrinsically possess structure at various different scales, as in
Figure \ref{multiscale} below. Clustering techniques should reflect
this structure, and provide methods for representing and analyzing it.

	\begin{figure}[htb] \centering
	\includegraphics[width=.75\linewidth]{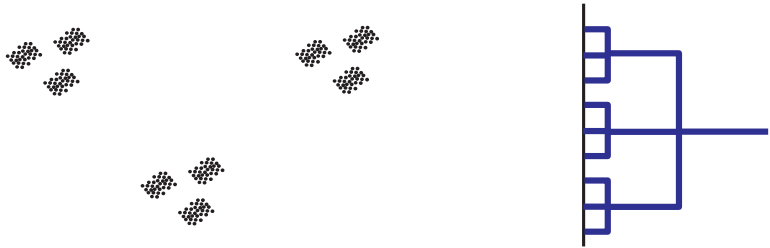}
	\caption{Dataset with multiscale structure and its
	corresponding dendrogram.}  \label{multiscale} \end{figure}

\noindent Ideally, users should be presented with a readily
computable and presentable object which will give him/her the option
of choosing the proper scale for the analysis, or perhaps interpreting
the multiscale invariant directly, rather than being asked to choose a
scale or choosing it for him/her. It is widely accepted that
clustering is ultimately itself a tool for exploratory data analysis,
\cite{ulrike_general}. In some sense, it is therefore totally acceptable to 
provide this multiscale invariant, whenever available and let the user
pick different scale thresholds that will yield different partitions
of the data. Once we accept this, we can concentrate on answering
theoretical questions regarding schemes that output this kind of
information. Our analysis will not, however, rule out clustering
methods that provide a one-scale view of the data, since, formally,
one can consider a such a scheme as one that at all scales gives the
same information, cf. Example \ref{example:scale-free}

We choose a particular way of representing this multiscale
information, we use the formalism of
\emph{persistent sets}, which is introduced in Section
\ref{persistence}, Definition \ref{def:persistent-sets}. The idea of showing the multiscale clustering view
of the dataset is widely used in Gene expression data analysis and it
takes the form of \emph{dendrograms}.}

\medskip
\noindent {{\bf Functoriality:} As our replacement for the
constraints discussed in \cite{kleinberg}, we will use instead the
notion of {\em functoriality} which has been a very useful framework
for the discussion of a variety of problems within mathematics over
the last few decades. For a discussion of categories and functors, see
\cite{maclane}.  Our idea is that clusters should be viewed as the
stochastic analogue of the mathematical concept of {\em path
components}.  Recall (see, e.g. \cite{munkres}) that the path
components of a topological space $X$ are the equivalence classes of
points in the space under the equivalence relation $\sim _{path}$,
where, for $x,y \in X$, we have $x \sim _{path} y$ if and only if
there is a continuous map $\varphi: [0,1] \rightarrow X$ so that
$\varphi (0) = x$ and $\varphi (1) = y$.  In other words, two points
in $X$ are in the same path component if they are connected by a
continuous path in $X$.  This set of components is denoted by $\pi _0
(X)$.  The assignment $X \rightarrow \pi _0 (X)$ is said to be
functorial, in that given a continuous map $f :X \rightarrow Y$
(\emph{morphism} of topological spaces), there is a natural map of
sets $\pi _0 (f) : \pi _0(X) \rightarrow \pi _0 (Y)$, which is defined
by requiring that $\pi _0 (f)$ carries the path component of a point
$x \in X$ to the path component of $f(x) \in Y$.  This notion has been
critical in many aspects of geometry; it provides the basis for the
methods of organizing geometric objects combinatorially which is
referred to as combinatorial or simplicial topology.  

\smallskip
The input to clustering algorithms is not, of course, a topological
space.  Rather, it is typically {\em point cloud data}, finite sets of
points lying in a Euclidean space of some dimension, or perhaps in
some other metric space, such as a tree or a collection of words in
some alphabet equipped with a metric.  We will therefore think of it
as a finite metric space (see \cite{munkres} for a discussion of
metric spaces).  There is a natural notion of what is meant by a map
of metric spaces, which one can think of as loosely analogous to
continuity. This notion has been used in other contexts in the past,
see for example
\cite{isbell}. Similarly, we define a natural notion of what is meant
by a morphism of the persistent sets defined above, and require
\emph{functoriality} for the clustering algorithms we consider in terms
of these notions of morphisms. For the time being the reader not
familiar with the concept, can think of functoriality as a notion of
coarse stability/consistency.  By varying the richness of the class of
morphisms between metric spaces one can control how stringent are the
conditions imposed on the clustering algorithms. Functoriality can
therefore be interpreted as a notion of \emph{coarse stability} of
these clustering algorithms.} 

\medskip
In \cite{mccullaugh}, the idea of using categorical and functorial ideas in statistics has been proposed as a formalism for defining what is meant by statistical models.   One aspect of our work is to show that the same ideas, which are so powerful in many other aspects of mathematics, can be used to understand the nature of algorithms for accomplishing statistical tasks.

\medskip
 We summarize the main features of our point of view.
	
	(a) It makes explicit the notion of multiscale representation
	of the set of clusters.

	(b) By varying the degree of functoriality (i.e. by
	considering different notions of morphism on the domain of
	point cloud data) one can reason about the existence and
	properties of various schemes.  We illustrate this possibility
	in Section \ref{results}. In particular, are able to prove a
	\emph{uniqueness} theorem for clustering algorithms with one
	natural notion of functoriality.

	(c) Beyond the conceptual advantages cited above,
	functoriality can be directly useful in analyzing datasets.
	The property can be used to study qualitative geometric
	properties of point cloud data, including more subtle
	geometric information than clustering, such as presence of
	``loopy" behavior or higher dimensional analogues.  See
	e.g. \cite{mumford} for an example of this point of view.  We
	will also present an example in Subsection \ref{intrinsic}.
	In addition, the functoriality property can be used to analyze
	functions on the datasets, by studying the behavior of
	sublevel sets of the function under clustering.  One version
	of this idea builds probabilistic versions of the {\em Reeb
	graph}.  See \cite{mapper} for a number of examples of how
	this can work.

Other, different, notions of stability of clustering schemes have
appeared in the literature, see \cite{raghavan,sober} and references
therein. We touch upon similar concepts in Section \ref{sec:stab}.

The organization of the paper is as follows. In Section
\ref{persistence} we introduce the main objects that model the output
of clustering algorithms together with some important
examples. Section \ref{cats} introduces the concepts of categories and
functors, and the idea of \emph{functoriality} is discussed. We
present our main characterization results in Section
\ref{results}. The quantitative study of stability and consistency is presented in Section \ref{sec:stab}. Further applications of the concept of functoriality are discussed in Section \ref{sec:zigzags} and concluding remarks are presented in Section \ref{various}.

\section{Persistence}\label{persistence}
In this section we define the objects which are the output of the
clustering algorithms we will be working with.  These objects will encode
the notion of ``multiscale" or ``multiresolution" sets discussed in
the introduction.

Let ${\mathcal P}(X)$ denote the set of partitions of the (finite) set
$X$.

\begin{Definition}\label{def:persistent-sets}
A {\em persistent set} is a pair $(X, \theta)$, where $X$ is a finite
set, and $\theta$ is a function from the non-negative real line $[0,+
\infty)$ to ${\mathcal P}(X)$  so that the following
properties hold.
\begin{enumerate}
\item{ If $r \leq s$, then $\theta (r)$ refines $\theta (s)$.  }
\item{For any $r$, there is a number $\epsilon > 0$ so that $\theta (r^{\prime}) = \theta (r)$ for all $r^\prime\in[r,r+\epsilon]$.} 
\end{enumerate}
If in addition there exists $t>0$ s.t. $\theta(t)$ consists of the
single block partition for all $r\geq t$, then we say that
$(X,\theta)$ is a \emph{dendrogram}.\footnote{In the paper we will be
using the word dendrogram to refer both to the object defined here and
to the standard graphical representation of them.}
\end{Definition}
\noindent The intuition is that the set of blocks of the partition $\theta (r)$ should be regarded as $X$ viewed at scale $r$.  

\begin{Example}\label{rips}{\em  Let $(X,d)$ be a finite metric space.  Then we can associate to $(X,d)$ the persistent set whose underlying set is $X$, and where blocks of  the partition $\theta (r)$ consist of the  equivalence classes under the equivalence relation $\sim _r$, where $x \sim _r x^{\prime}$ if and only if there is a sequence $x_0, x_1, \ldots , x_t \in X$ so that $x_0 = x, x_t = x^{\prime}$, and $d(x_i, x_{i+1}) \leq r$ for all $i$.}  
\end{Example}

\begin{Example}\label{example:scale-free}
{\em A more trivial example is one in which $\theta (r) $ is constant,
i.e. consists of a single partition.  This is the scale free notion of
clustering. Examples are $k$-means clustering and spectral
clustering.}
\end{Example}

\begin{Example}{\em \label{hierarchical}
Here we consider the family of Agglomerative Hierarchical clustering
techniques, \cite{clusteringref}. We (re)define these by the recursive
procedure described next. Let $X=\{x_1,\ldots,x_n\}$ and let $\mathcal
L$ denote a family of \emph{linkage functions}, i.e. functions which
one uses for defining the distance between two clusters. Fix $l\in
{\mathcal L}$. For each $R>0$ consider the equivalence relation
$\sim_{l,R}$ on blocks of a partition $\Pi\in{\mathcal P}(X)$, given
by \mbox{${\mathcal B}\sim_{l,R} {\mathcal B}'$} if and only if there is a
sequence of blocks ${\mathcal B}={\mathcal B}_1,\ldots,{\mathcal
B}_s={\mathcal B}'$ in $\Pi$ with $l({\mathcal B}_k,{\mathcal
B}_{k+1})\leq R$ for $k=1,\ldots,s-1$. Consider the sequences
$r_1,r_2,\ldots\in [0,\infty)$ and $\Theta_1,\Theta_2,\ldots\in
\mathcal{P}(X)$ given by $\Theta_1:=\{x_1,\ldots,x_n\}$ and for $i\geq
1$, $\Theta_{i+1}=\Theta_i/\sim_{l,r_i}$ where $r_i:=
\min\{l({\mathcal B},{\mathcal B}'),\,{\mathcal B},{\mathcal B}'\in\Theta_i,\,{\mathcal B}\neq {\mathcal B}'\}.$
Finally, we define $\theta^l:[0,\infty)\rightarrow \mathcal{P}(X)$ by
$r\mapsto
\theta^l(r):=\Theta_{i(r)}$ where $i(r):=\max\{i|r_i\leq
r\}$. Standard choices for $l$ are \emph{single linkage}: $l({\mathcal
B},{\mathcal B}')=\min_{x\in {\mathcal B}}\min_{x'\in {\mathcal
B}'}d(x,x');$
\emph{complete linkage}
$l({\mathcal B},{\mathcal B}')=\max_{x\in {\mathcal B}}\max_{x'\in
{\mathcal B}'}d(x,x');$ and \emph{average linkage}: $l({\mathcal
B},{\mathcal B}')=\frac{\sum_{x\in {\mathcal B}}\sum_{x'\in {\mathcal
B}'}d(x,x')}{\#{\mathcal B}
\cdot \#{\mathcal B}'}.$ It is easily verified that the notion discussed in
Example \ref{rips} is \emph{equivalent} to $\theta^l$ when $l$ is the
single linkage function. Note that, unlike the usual definition of
agglomerative hierarchical clustering, at each step of the inductive
definition we allow for more than two clusters to be merged.}
\end{Example}

\noindent   We will be using the persistent sets which arise out of Example \ref{rips}.  
It is of course the case that the persistent set carries much more
information than a single set of clusters.  One can ask whether it
carries too much information, in the sense that either (a) one cannot
obtain useful interpretations from it or (b) it is computationally
intractable.  We claim that it can usually be usefully interpeted, and
can be effectively and efficiently computed.  One can observe this as
follows.  Since there are only a finite number of partitions of $X$, a
persistent set ${\mathcal Q}$ gives a partition of $\mathbb{R}^+$ into a
finite collection ${\mathcal I}$ of intervals of the form
$[r,r^{\prime})$, together with one interval of the form $[r, +
\infty)$.  For each such interval, every number in the interval
corresponds to the same partition of $X$.  

We claim that knowledge of
these intervals is a key piece of information about the persistent
sets arising from Examples \ref{rips} and \ref{hierarchical} above.
The reason is that long intervals in ${\ I}$ correspond to large
ranges of values of the scale parameter in which the associated
cluster decomposition doesn't change.  One would then regard the
partition into clusters corresponding to that interval as likely to
represent significant structure present at the given range of scales.
If there is only one long interval (aside from the infinite interval
of the form $[r, +
\infty)$) in ${\mathcal I}$, then one is led to believe that there is only
one interesting range of scales, with a unique decomposition into
clusters.  However, if there are more that one long interval, then it
suggests that the object has significant multiscale behavior, see
Figure \ref{multiscale}.  Of course, the determination of what is
``long" and what is ``short" will be problem dependent, but choosing
thresholds for the length of the intervals will give definite ranges
of scales.  As for the computability, the persistent sets associated
to a finite metric space can be readily computed using (conveniently
modified) hierarchical clustering techniques, or the methods of
persistent homology (see
\cite{persistence}).


\section{Categories, functors and functoriality}\label{cats}
\subsection{Definitions and Examples}
In this section, we will give a brief description of the theory of
categories and functors, which will be the framework in which we state
the constraints we will require of our clustering algorithms.  An
excellent reference for these ideas is \cite{maclane}.

Categories are useful mathematical constructs that encode the nature
of certain objects of interest together with a set of
admissible/interesting/useful maps between them. This formalism is
extremely useful for studying classes of mathematical objects which
share a common structure, such as sets, groups, vector spaces, or
topological spaces.  The definition is as follows.

\begin{Definition} A category $\underline{C}$ consists of 
\begin{itemize}

	\item{ A collection of objects $ob(\underline{C})$ (e.g.
	sets, groups, vector spaces, etc.)}

	\item{ For each pair of objects $X,Y \in ob(\underline{C})$, a
	set\\
	 $Mor_{\underline{C}} (X,Y)$, the morphisms from $X$ to $Y$
	(e.g. maps of sets from $X$ to $Y$, homomorphisms of groups
	from $X$ to $Y$, linear transformations from $X$ to $Y$,
	etc. respectively)}
	
	\item{ Composition operations:\\ $\compcirc :
	Mor_{\underline{C}} (X,Y) \times Mor_{\underline{C}} (Y,Z)
	\rightarrow Mor_{\underline{C}} (X,Z)$, corresponding to
	composition of set maps, group homomorphisms, linear
	transformations, etc.  }
	
	\item{For each object $X \in \underline{C}$, a distinguished element $id_X\in Mor_{\underline{C}}(X,X)$}
\end{itemize}
The composition is assumed to be associative in the obvious sense, and for any $f \in Mor_{\underline{C}} (X,Y)$, it is assumed that $id_Y \compcirc f = f$ and $f \compcirc id_X = f$. 
\end{Definition} 
Here are the relevant examples for this paper. 
\begin{Example}
{\em We will construct three categories $\underline{{\mathcal M}}^{iso}$, $\underline{{\mathcal M}}^{mon}$, and $\underline{{\mathcal M}}^{gen}$, whose
collections of objects will all consist of the collection of finite
metric spaces.  Let $(X,d_X)$ and $(Y, d_Y)$ denote finite metric
spaces.  A set map $f: X \rightarrow Y$ is said to be {\em distance
non increasing} if for all $x,x^{\prime} \in X$, we have
$d_Y(f(x),f(x^{\prime})) \leq d_X(x,x^{\prime})$.  It is easy to check
that composition of distance non-increasing maps are also distance
non-increasing, and it is also clear that $id_X$ is always distance
non-increasing.  We therefore have the category $\underline{{\mathcal
M}}^{gen}$, whose objects are finite metric spaces, and so that for
any objects $X$ and $Y$, $Mor_{\underline{{\mathcal M}}^{gen}}(X,Y)$ is
the set of distance non-increasing maps from $X$ to $Y$, cf.
\cite{isbell} for another use of this class of maps.  We say that a distance non-increasing map is {\em
monic} if it is an inclusion as a set map.  It is clear compositions
of monic maps are monic, and that all identity maps are monic, so we
have the new category $\underline{{\mathcal M}}^{mon}$, in which
$Mor_{\underline{{\mathcal M}}^{mon}}(X,Y)$ consists of the monic distance
non-increasing maps.  Finally, if $(X,d_X)$ and $(Y,d_Y)$ are finite
metric spaces, $f:X
\rightarrow Y$ is an {\em isometry} if $f$ is bijective and $d_Y(f(x),
f(x^{\prime})) = d_X(x,x^{\prime})$ for all $x$ and $x^{\prime}$.  It
is clear that as above, one can form a category $\underline{{\mathcal
M}}^{iso}$ whose objects are finite metric spaces and whose morphisms
are the isometries.  It is clear that we have inclusions 
\begin{equation}\label{eq:inclu-cats}
\underline{{\mathcal M}}^{iso} \subseteq \underline{{\mathcal M}}^{mon}
\subseteq \underline{{\mathcal M}}^{gen} 
\end{equation}
 of subcategories (defined as in \cite{maclane}).  Note that although
 the inclusions are bijections on object sets, they are proper
 inclusions on morphism sets, i.e. they are not in general surjective.
 }
\end{Example}
We will also construct a category of \emph{persistent sets}. 

\begin{Example}\label{example:P}
{\em Let $(X, \theta), (Y, \eta)$ be persistent sets.  For any
partition $\Pi$ of a set $Y$, and any set map $f: X \rightarrow Y$, we
define $f^*(\Pi)$ to be the partition of $X$ whose blocks are the sets
$f^{-1}({\mathcal B})$, as ${\mathcal B}$ ranges over the blocks of $\Pi$.  A
map of sets $f: X \rightarrow Y$ is said to be {\em persistence
preserving} if for each $r \in \mathbb{R}$, we have that $\theta (r)$
is a refinement of $ f^*(\eta(r))$.  It is easily verified that the
composite of persistence preserving maps is persistence preserving,
and that any identity map is persistence preserving, and it is
therefore clear that we may define a category $\underline{{\mathcal P}}$
whose objects are persistent sets, and where $Mor_{\underline{{\mathcal
P}}}((X,\theta), (Y, \eta))$ consists of the set maps from $X$ to $Y$
which are persistence preserving. A simple example is shown in Figure
\ref{fig:pers-preserv}.  }

\begin{figure}[htb]
	\centering
	\includegraphics[width=1\linewidth]{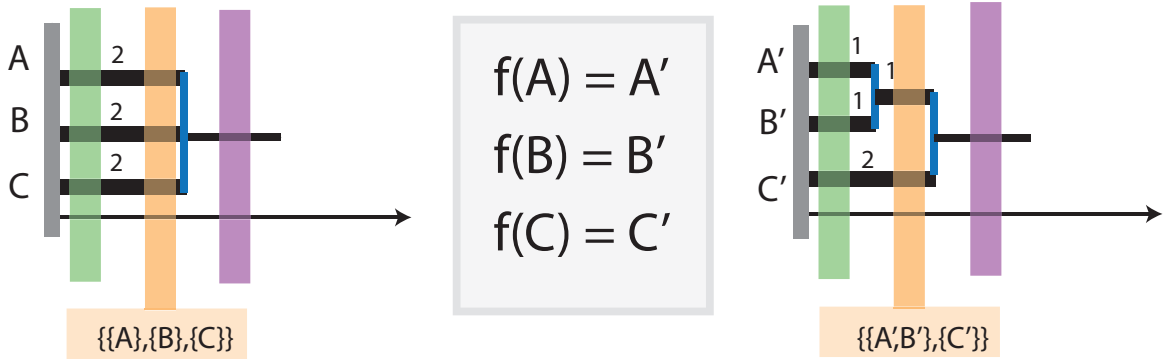}
	\caption{Two persistent sets $(X,\theta)$ and $(Y,\eta)$
	represented by their dendrograms. On the left one defined in
	the set $X = \{A,B,C\}$ and on the right one defined on the
	set $Y=\{A',B',C'\}$. Consider the given set map
	$f:X\rightarrow Y$. Then we see that $f$ is persistence
	preserving since for each $r\geq 0$, the partition $\theta(r)$
	is a refinement of $f^*(\eta(r))$. Indeed, there are three
	interesting ranges of values of $r$. Pick for example $r$ like
	in the orange shaded area: $r\in[1,2)$. Then $\eta(r) =
	\{\{A',B'\},\{C'\}\}$ and hence $f^*(\eta(r)) =
	\{f^{-1}(\{A',B'\}),\{f^{-1}(C')\}\} = \{\{A,B\},\{C\}\}$
	which is indeed refined by
	$\theta(r)=\{\{A\},\{B\},\{C\}\}$. One proceeds similary for
	the other two cases.}  \label{fig:pers-preserv}
\end{figure}
\end{Example}
 We next introduce the key concept in our discussion, that of a
 \emph{functor}.  We give the formal definition first.
\begin{Definition}\label{def:functor}
Let $\underline{C}$ and $\underline{D}$ be categories.  Then a {\em
functor} from $\underline{C}$ to $\underline{D}$ consists of

\begin{itemize}

	\item{A map of sets $F:ob(\underline{C}) \rightarrow
	ob(\underline{D})$} 
	
	\item{For every pair of objects $X,Y \in \underline{C}$ a map
	of sets $\Phi(X,Y):Mor_{\underline{C}}(X,Y) \rightarrow
	Mor_{\underline{D}}(FX,FY)$ so that

	\begin{enumerate}

		\item{ $\Phi(X,X)(id_X) = id_{F(X)}$ for all $X \in
		ob(\underline{C})$}

		\item{$ \Phi (X,Z)(g \compcirc f) = \Phi (Y,Z)(g)
		\compcirc \Phi (X,Y)(f)$ for all $f \in Mor
		_{\underline{C}}(X,Y)$ and $g \in Mor
		_{\underline{C}}(Y,Z)$}

	\end{enumerate} }
\end{itemize}
\end{Definition}
\begin{Remark}
In the interest of clarity, we often refer to the pair $(F,\Phi)$ as a
single letter $F$. See diagram (\ref{eq:diagram}) in Example
\ref{ex:key} below for an example.
\end{Remark}
A morphism $f:X \rightarrow Y$ which has a two sided inverse $g:
Y\rightarrow X$, so that $f \compcirc g = id_Y$ and $g \compcirc f =
id_X$, is called an {\em isomorphism}.  Two objects which are
isomorphic are intuitively thought of as ``structurally
indistinguishable" in the sense that they are identical except for
naming or choice of coordinates.  For example, in the category of
sets, the sets $\{ 1,2,3 \}$ and $\{ A,B,C \}$ are isomorphic, since
they are identical except for choice made in labelling the elements.
We illustrate this definition with some examples.
\begin{Example}{\bf (Forgetful functors)} {\em When one has two categories $\underline{C}$ and $\underline{D}$, where the objects in $\underline{C}$ are objects in $\underline{D}$  equipped with some additional structure and the morphisms in $\underline{C}$ are simply the morphisms in $\underline{D}$ which preserve that structure, then we obtain the ``forgetful functor" from $\underline{C}$ to $\underline{D}$, which carries the object in $\underline{C}$ to the same object in $\underline{C}$, but regarded without the additional structure.  For example, a group can be regarded as a set with the additional structure of multiplication and inverse maps, and the group homomorphisms are simply the set maps which respect that structure.  Accordingly, we have the functor from the category of groups to the category of sets which ``forgets the multiplication and  inverse".  Similarly, we have the forgetful functor from $\underline{{\mathcal P}}$ to the category of sets, which forgets the  presence of $\theta$ in the persistent set $(X, \theta)$.  }
\end{Example}
\begin{Example} {\em  
The inclusions $ \underline{{\mathcal M}}^{iso} \subseteq \underline{{\mathcal
M}}^{mon} \subseteq \underline{{\mathcal M}}^{gen} $ are both functors.  }
\end{Example}
Any given clustering scheme is a procedure $F$ which takes as input a
finite metric space $(X,d_X)$, that is, an object in $ob(\underline{\mathcal
M}^{gen})$, and delivers as output a persistent set, that is, an
object in $ob(\underline{\mathcal P})$. The concept of
\underline{functoriality} refers to the additional condition
that the clustering procedure maps a pair of input objects into a pair
of output objects in a manner which is consistent/stable with respect
to the morphisms attached to the input and output spaces. When this
happens, we say that the clustering scheme is \underline{functorial}.
This notion of consistency/stability is made precise in Definition
\ref{def:functor} and described by diagram (\ref{eq:diagram}).

Now, the idea is to regard clustering algorithms (that output a
persistent set) as functors. Assume for instance we want to consider
``stability'' to all distance non-increasing maps. Then the correct
category of inputs (finite metric spaces) is $\underline{{\mathcal
M}}^{gen}$ and the category of outputs is $\underline{\mathcal P}$.
According to Definition \ref{def:functor} in order to view a
clustering scheme as a functor we need to specify (1) how it maps
objects of $\underline{{\mathcal M}}^{gen}$ (finite metric spaces)
into objects of $\underline{\mathcal P}$ (persistent sets), and (2)
how a valid morphism/map $f:(X,d_X)\rightarrow (Y,d_Y)$ between two
objects $(X,d_X)$ and $(Y,d_Y)$ in the input space/category
$\underline{{\mathcal M}}^{gen}$ induce a map in the output category
$\underline{\mathcal P}$, see diagram (\ref{eq:diagram}) below.

We exemplify this through the construction of the key example for this
paper.

\begin{Example} \label{ex:key}{\em We define a functor 
$${\mathcal R}^{gen}: \underline{{\mathcal M}}^{gen} \rightarrow
\underline{{\mathcal P}}$$
as follows.  For a finite metric space $(X, d_X)$, we define ${\mathcal
R}^{gen}(X,d_X)$ to be the persistent set $(X, \theta ^{d_X})$, where
$\theta ^{d_X}(r)$ is the partition associated to the equivalence
relation $\sim _r$ defined in Example \ref{rips}.  This is clearly an
object in $\underline{{\mathcal P}}$.  We also define how ${\mathcal
R}^{gen}$ acts on maps $f : (X,d_X) \rightarrow (Y,d_Y)$: The value of
${\mathcal R}^{gen}(f)$ is simply the set map $f$ regarded as a
morphism from $(X,\theta ^{d_X})$ to $(Y,\theta ^{d_Y})$ in
$\underline{{\mathcal P}}$.  That it is a morphism in $\underline{{\mathcal
P}}$ is easy to check. This functorial construction is represented through the diagram below:
\begin{equation}\label{eq:diagram}
\xymatrix{
(X,d_X) \ar[r]^{{\mathcal R}^{gen}} \ar[d]_{f} & (X,\theta^{d_X}) \ar[d]^{{\mathcal R}^{gen}(f)} \\
(Y,d_Y) \ar[r]^{{\mathcal R}^{gen}} & (Y,\theta^{d_Y})} 
\end{equation}
where ${\mathcal R}^{gen}(f)$ is persistence preserving.}
\end{Example}

\begin{Example}{ \em  By restricting ${\mathcal R}^{gen}$ to the subcategories 
$\underline{{\mathcal M}}^{iso} \mbox{ and } \underline{{\mathcal M}}^{mon} $,
we obtain functors\\ ${\mathcal R}^{iso}: \underline{{\mathcal M}}^{iso}
\rightarrow \underline{{\mathcal P}}$ and ${\mathcal R}^{mon}:
\underline{{\mathcal M}}^{mon} \rightarrow \underline{{\mathcal P}}$.}
\end{Example}
\begin{Example}\label{scaling} {\em Let $\lambda $ be any positive real number.  Then we define a functor $\sigma _{\lambda} : \underline{{\mathcal M}}^{gen}
\rightarrow \underline{{\mathcal M}}^{mon}$ on objects by 
$$  \sigma _{\lambda}(X,d_X) = (X,\lambda d_X)
$$
and on morphisms by $\sigma _{\lambda} (f) = f$.  One easily verifies that if $f$ satisfies the conditions for being a morphism in $\underline{{\mathcal M}}^{gen}$ from $(X,d_X)$ to $(Y,d_Y)$, then it readily satisfies the conditions to be a morphism from $(X,\lambda d_X)$ to $(Y, \lambda d_Y)$.  Similarly, we define a functor $s_{\lambda} : \underline{{\mathcal P}} \rightarrow \underline{{\mathcal P}}$ by setting $s_{\lambda}(X, \theta ) = (X, \theta ^{\lambda})$, where 
$\theta ^{\lambda}(r) = \theta (\frac{r}{\lambda})$. } 
\end{Example}

In Section \ref{sec:results} we will be showing our main results. We
will now have a brief disgression to discuss other situations in
which, in our opinion, the concept of functoriality can be useful.

\subsection{Intrinsic Value of Functoriality}\label{intrinsic}
\noindent By studying functorial methods of clustering, it is possible to
 recover qualitative aspects of the geometric structure of a dataset.
 We illustrate this idea with a ``toy" example.  We suppose that we
 have a point cloud data which is concentrated around the unit circle.
 We consider the projection of the data on to the $x$-axis, and cover
 the axis with two (overlapping) intervals $U$ and $V$, pictured on
 the left in Figure \ref{hocolim} below as being red and yellow, with
 orange intersection.  By considering those portions of the dataset
 whose $x$-coordinate lie in $U$ and $V$ respectively, we obtain the
 red and yellow subsets of the dataset pictured on the right below.
 Their intersection is pictured as orange, and the arrows indicate
 that we have inclusions of the intersection into each of the pieces.
\begin{figure}[htb]
	\centering \includegraphics[width=1\linewidth]{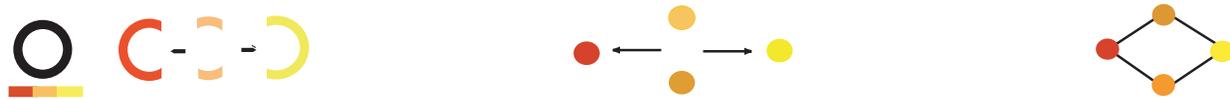}
	\caption{\emph{Left}: Covering of a circle by two
	intervals. \emph{Center}: Corresponding diagram of
	components. \emph{Right}: Homotopy colimit of diagram in
	center figure.}  \label{hocolim}
\end{figure}
\noindent Next, we note that if we are dealing with a functorial 
clustering scheme, and cluster each of these subsets, we obtain the
diagram of clusters in the center of Figure \ref{hocolim}.  This is
now a very simple combinatorial object.

\noindent There is a topological construction known as the {\em homotopy colimit}, which given any diagram of sets of any shape reconstructs a simplicial set (a slightly more flexible version of the notion of simplicial complex), and in particular a space.  To first approximation, one builds a vertex for every element in any set in the diagram, and an edge between any two elements which are connected by a map in the diagram, and then attaches higher order simplices according to a well defined procedure.  In the case of the diagram above, this constructs the space given in the rightmost part of Figure \ref{hocolim} . 

\noindent The details of the theory of simplicial sets and homotopy colimits are beyond the scope of this paper.  A thorough exposition is given in \cite{bk}. 

\noindent Functoriality is also quite useful when one is interested in studying the 
qualitative behavior of a real-valued function $f$ on a dataset, for
example the output of a density estimator.  Then it is useful to study
the set of clusters in sublevel and superlevel sets of $f$, and
understanding how the clusters behave under changes in the thresholds
can help one understand the presence of saddle points and higher index
critical points of the function. 

One example of this is two-parameter persistence constructions,
\cite{gsan}. In this case, there is more structure than just persistent sets (trees/dendrograms) as defined in this paper.

We will elaborate on another application of functoriality in Section
\ref{bootstrap}.

\section{Results}\label{results} \label{sec:results}

We now study different clustering algorithms using the idea of
functoriality.  We have 3 possible ``input'' categories ordered by
inclusion (\ref{eq:inclu-cats}). The idea is that studying
functoriality over a larger category will be more stringent/demanding
than requiring functoriality over a smaller one. We now consider
different clustering algorithms and study whether they are functorial
over our choice of the input category. We start by analyzing
functoriality over the least demanding one, $\underline{\mathcal
M}^{iso}$, then we prove a uniqueness result for functoriality over
$\underline{\mathcal M}^{gen}$ and finally we study how relaxing the
conditions imposed by the morphisms in $\underline{\mathcal M}^{gen}$,
namely, by restricting ourselves to the smaller but intermediate
category $\underline{\mathcal M}^{mon}$, we permit more functorial
clustering algorithms.

\subsection{Functorality over $\underline{{\mathcal M}}^{iso}$} 
This is the smallest category we will deal with.  The morphisms in
$\underline{{\mathcal M}}^{iso}$ are simply the bijective maps between
datasets which preserve the distance function.  As such, functoriality
of a clustering algorithms over $\underline{{\mathcal M}}^{iso}$ simply means
that the output of the scheme doesn't depend on any artifacts in the
dataset, such as the way the points are named or the way in which they
are ordered.  Here are some examples which illustrate the idea.
\begin{itemize}

	\item{The {\em $k$-means algorithm} (see \cite{clusteringref})
	is in principle allowed by our framework since
	$ob(\underline{\mathcal P})$ contains all constant persistent
	sets. However it is not functorial on any of our input
	categories.  It depends both on a paramter $k$ (number of
	clusters) and on an initial choices of means, and is not
	therefore dependent on the metric structure alone.}

	\item{{\em Agglomerative hierarchical clustering}, in standard
	form, as described for example in \cite{clusteringref}, begins
	with point cloud data and constructs a \emph{binary} tree (or
	dendrogram) which describes the merging of clusters as a
	threshold is increased.  The lack of functoriality comes from
	the fact that when a single threshold value corresponds to
	more than one data point, one is forced to choose an ordering
	in order to decide which points to ``agglomerate" first.  This
	can easily be modified by relaxing the requirement that the
	tree be binary. This is what we did in Example
	\ref{hierarchical} In this case, one can view these methods as
	functorial on $\underline{{\mathcal M}}^{iso}$, where the functor
	takes its values in arbitrary rooted trees.  It is understood
	that in this case, the notion of morphism for the output
	($\underline{\mathcal P}$) is simply isomorphism of rooted
	trees. In contrast, we see next that amongst these methods,
	when we impose that they be functorial over the larger (more
	demanding) category $\underline{\mathcal M}^{gen}$ then only
	one of them passes the test.\footnote{The result in Theorem
	\ref{thm:uniq} is actually more powerful in that it states
	that there is a unique functor from $\underline{\mathcal
	M}^{gen}$ to $\underline{\mathcal P}$ that satisfies certain
	natural conditions.} }

	\item{{\em Spectral clustering}. As described in
	\cite{ulrike_spectral}, typically, spectral methods consist of
	two different layers. They first define a laplacian matrix out
	of the dissimilarity matrix (given by $d_X$ in our case)} and
	then find eigenvalues and eigenvectors of this operator. The
	second layer is as follows: a natural number $k$ must be
	specified, a projection to $\R^k$ is performed using the
	eigenfunctions, and clusters are found by an application of
	the $k$-means clustering algorithm. Clearly, operations in the second layer
	will fail to be functorial as they do not depend on the metric
	alone. However, the procedure underlying in the first layer is
	clearly functorial on $\underline{\mathcal M}^{iso}$ as
	eigenvalue computations are changed by a permutation in a well
	defined, natural, way. 
\end{itemize}

\subsection{Functorality over $\underline{\mathcal M}^{gen}$: a uniqueness theorem}
In this section, as an example application of the conceptual framework
of functoriality, we will prove a theorem of the same flavor as the
main theorem of \cite{kleinberg}, except that we prove existence and
uniqueness on $\underline{\mathcal M}^{gen}$ instead of impossibility
in our context.

Before stating and proving our theorem, it is interesting to point out
why complete linkage and average linkage (agglomerative) clustering,
as defined in Example \ref{hierarchical} are not functorial on
$\underline{\mathcal M}^{gen}$. A simple example explains this: consider
the metric spaces $X=\{A,B,C\}$ with metric given by the edge lengths
$\{4,3,5\}$ and $Y=(A',B',C')$ with metric given by the edge lengths
$\{4,3,2\}$, as given in Figure
\ref{cexample}. Obviously the map $f$ from $X$ to $Y$ with $f(A)=A'$, 
$f(B)=B'$ and $f(C)=C'$ is a morphism in $\underline{\mathcal
M}^{gen}$. Note that for example for $r=3.5$ (shaded regions of the
dendrograms in Figure \ref{cexample}) we have that the partition of
$X$ is $\Pi_X = \{\{A,C\},B\}$ whereas the partition of $Y$ is $\Pi_Y
= \{\{A',B'\},C'\}$ and thus $f^*(\Pi_Y)=\{\{A,B\},\{C\}\}$. Therefore
$\Pi_X$ does not refine $f^*(\Pi_Y)$ as required by functoriality. The same
construction yields a counter-example for average linkage.
\begin{figure}[htb]
		\centering
		\includegraphics[width=.7\linewidth]{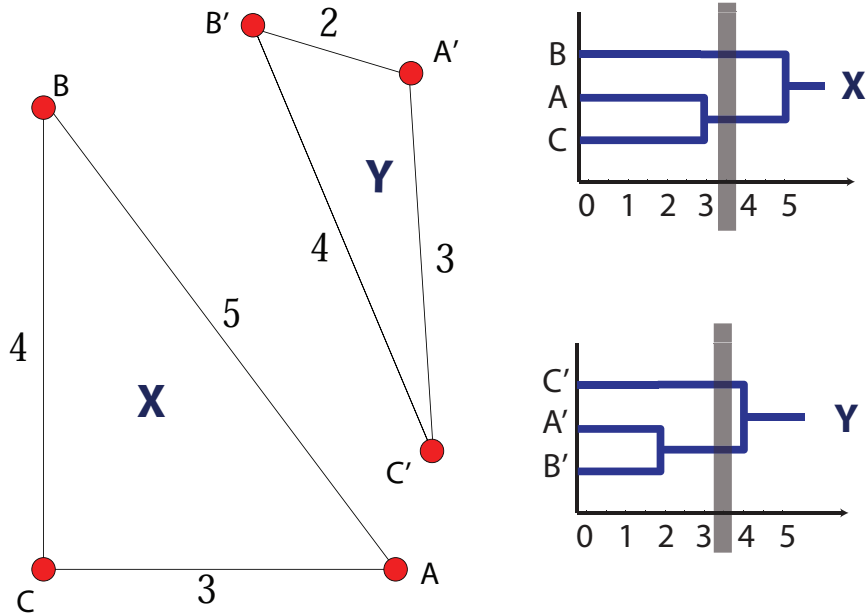}
		\caption{An example that shows why complete linkage
		fails to be functorial on $\underline{\mathcal M}^{gen}$.}
		\label{cexample}
\end{figure}
\begin{Theorem} \label{thm:uniq} Let $\Psi: \underline{{\mathcal M}}^{gen} \rightarrow \underline{{\mathcal P}} $
be a functor which satisfies the following conditions.
\begin{description}

	\item[(I)] {Let $\alpha : \underline{{\mathcal M}}^{gen} \rightarrow
	\underline{Sets}$ and $\beta :\underline{ {\mathcal P}}\rightarrow
	\underline{Sets}$ be the forgetful functors$ (X,d_X)
	\rightarrow X$ and $(X,\theta) \rightarrow X$, which forget
	the metric and partition respectively, and only ``remember"
	the underlying sets $X$.  Then we assume that $\beta \compcirc
	\Psi = \alpha$.  This means that the underlying set of the
	persistent set associated to a metric space is just the
	underlying set of the metric space. }

	\item[(II)] {For $\delta\geq 0$ let
	$Z(\delta)=(\{p,q\},\dmtwo{\delta})$ denote the two point
	metric space with underlying set $\{ p,q \}$, and where
	$\mbox{dist}(p,q) = \delta$.  Then $\Psi (Z(\delta))$ is the
	persistent set $(\{p,q\},\theta ^{Z(\delta)})$ whose
	underlying set is $\{p,q \}$ and where $\theta ^{Z(\delta)}
	(t) $ is the partition with one element blocks when $t <
	\delta$ and it is the partition with a single two point block
	when $t \geq \delta$.  }

	\item[(III)] {Given a finite metric space $(X,d_X)$, let
	$$\mbox{sep}(X) := \min_{x \neq
	x^{\prime}}d_X(x,x^{\prime}).$$ Write
	$\Psi(X,d_X)=(X,\theta^{\Psi})$, then for any $t <
	\mbox{sep}(X)$, the partition $\theta^{\Psi}(t)$ is the
	discrete partition with one element blocks.  }
\end{description}
Then $\Psi $ is equal to the functor ${\mathcal R}^{gen}$.  
\end{Theorem}

\begin{proof}
Let $\Psi(X,d_X) = (X,\theta^{\Psi})$. For each $r\geq 0$ we will
prove that \textbf{(a)} $\theta^{d_X}(r)$ is a refinement of $\theta^{\Psi}(r)$
and \textbf{(b)} $\theta^{\Psi}(r)$ is a refinement of $\theta^{d_X}(r)$.

Then it will follow that $\theta^{d_X}(r)=\theta^{\Psi}(r)$ for all
$r\geq 0$, which shows that the objects are the same.  Since this is a situation where, given any pair of objects, there is at most one  morphism    between them,  this also determines the effect of the functor on morphisms.  


Fix $r\geq 0$. In order to obtain (a) we need to prove that whenever
$x,x'\in X$ lie in the same block of the partition $\theta^{d_X}(r)$,
that is $x\sim_r x'$, then they both lie in the same block of
$\theta^{\Psi}(r)$.

\smallskip
It is enough to prove the following
\underline{\textbf{Claim}}:  whenever\\
 $d_X(x,x')\leq r$ then $x$ and $x'$ lie in the same block of $\theta^{\Psi}(r)$.

\smallskip
Indeed, if the claim is true, and $x\sim_r x'$ then one can find
$x_0,x_1,\ldots,x_n$ with $x_0=x$, $x_n=x'$ and $d_X(x_i,x_{i+1})\leq
r$ for $i=0,1,2,\ldots,n-1$. Then, invoking the claim for all pairs
$(x_i,x_{i+1})$, $i=0,\ldots,n-1$ one would find that: ${x}=x_0$ and
$x_1$ lie in the same block of $\theta^{\Psi}(r)$, $x_1$ and
$x_2$ lie in the same block of $\theta^{\Psi}(r)$, $\ldots$, $x_{n-1}$
and $x_n=x'$ lie in the same block of $\theta^{\Psi}(r)$. Hence, $x$
and $x'$ lie in the same block of $\theta^{\Psi}(r)$.

So, let's prove the claim. Assume $d_X(x,x')\leq r$, then the function
given by $p \rightarrow x, q \rightarrow x^{\prime}$ is a morphism
$g:Z(r) \rightarrow ( X, d_X) $ in $\underline{{\mathcal M}}^{gen}$.
This means that we obtain a morphism $$\Psi (g) : \Psi (Z(r))
\rightarrow \Psi (X,d_X)$$ in $\underline{{\mathcal P}}$.  But, $p$ and
$q$ lie in the same block of the partition $\theta ^{Z(r)}$ by
definition of $Z(r)$, and functoriality therefore guarantees that
$\Psi(g)$ is persistence preserving (recall Example \ref{example:P})
and hence the elements $g(p) = x$ and $g(q) = x^{{\prime}}$ lie in the
same block of $ \theta ^{d_X} (r)$. This concludes the proof of (a).

\smallskip
For condition (b), assume that $x$ and $x'$ belong to the same block
of the partition $\theta^{\Psi}(r)$. We will prove that necessarily
$x\sim_r x'$. This of course will imply that $x$ and $x'$ belong to the
same block of $\theta^{d_X}(r)$.

\smallskip
Consider the metric space $(X[r],d_{[r]})$ whose points are the
equivalence classes of $X$ under the equivalence relation $\sim _r$,
and where the metric $d_{[r]}:X[r]\times X[r]\rightarrow \R^+$ 
 is defined to be the maximal metric pointwisely less than or equal to $W$, where for two points ${\mathcal B}$ and ${\mathcal B}' $ in $X[r]$ (equivalence classes of $X$ under $\sim_r$),
 $W({\mathcal B},{\mathcal B}') = \min_{x
\in {\mathcal B}}\min _{x^{\prime} \in {\mathcal B}'} d_X
(x,x^{\prime}).$\footnote{See Section \ref{sec:stab} for a similar,
explicit construction.} It follows from the definition of $\sim _r$
that if two equivalence classes are distinct, then the distance
between them is $> r$. This means that $\mbox{sep}(X[r])>r.$

Write $\Psi (X[r],d_{[r]})=(X[r],\theta^{[r]})$. Since
$\mbox{sep}(X[r])>r$, hypothesis (III) now directly shows that the
blocks of the partition $\theta^{[r]}(r)$ are exactly the equivalence
classes of $X$ under the equivalence relation $\sim _r$, that is
$\theta^{[r]}(r) = \theta^{d_X}(r)$. Finally, consider the morphism
$$\pi_r: (X,d_X)
\rightarrow (X[r],d_{[r]})$$ in $\underline{{\mathcal M}}^{gen}$ given on elements $x
\in X$ by $\pi_r (x) =[x]_r$, where $[x]_r$ denotes the equivalence class of
$x$ under $\sim_r$. By functoriality, $\Psi(\pi_r):(X,\theta^{\Psi})\rightarrow (X[r],\theta^{[r]})$ is
persistence preserving, and therefore, $\theta^{\Psi}(r)$ is a
refinement of $\theta^{[r]}(r)=\theta^{d_X}(r)$. This is depicted as follows: 
\begin{displaymath}
\xymatrix{
(X,d_X) \ar[r]^\Psi \ar[d]_{\pi_r} & (X,\theta^{d_X}) \ar[d]^{\Psi(\pi_r)} \\
(X[r],d_{[r]}) \ar[r]^{\Psi} & (X[r],\theta_{[r]}) }
\end{displaymath}

This concludes the proof of (b).
\end{proof}

\medskip
We should point out that another characterization of single linkage
has been obtained in the book \cite{math-taxo}.

\subsubsection{Comments on Kleinberg's conditions}\label{sec:klein-conds}
\noindent We conclude this section by observing that analogues of  the three (axiomatic) 
properties considered by Kleinberg in \cite{kleinberg} hold for ${\mathcal
R}^{gen}$.

	 Kleinberg's \underline{first condition} was {\em scale-invariance}, which
	 asserted that if the distances in the underlying point cloud
	 data were multiplied by a constant positive multiple
	 $\lambda$, then the resulting clustering decomposition should
	 be identical.  In our case, this is replaced by the condition
	 that ${\mathcal R}^{gen}\compcirc\sigma _{\lambda}(X,d_X) =
	 s_{\lambda}\compcirc{\mathcal R}^{gen}(X,d_X)$, which is
	 trivially satisfied.

	 Kleinberg's \underline{second condition}, {\em richness}, asserts that
	 any partition of a dataset can be obtained as the result of
	 the given clustering scheme for some metric on the
	 dataset. In our context, partitions are replaced by
	 persistent sets. Assume that there exist $t\in \Rr$
	 s.t. $\theta(t)$ is the single block partition, i.e., impose
	 that the persistent set is a dendrogram (cf. Definition
	 \ref{def:persistent-sets}).  In this case, it is easy to
	 check that any such persistent set can be obtained as ${\mathcal
	 R}^{gen}$ evaluated for some (pseudo)metric on some
	 dataset. Indeed,\footnote{We only prove triangle inequality.}
	 let $(X,\theta)\in{ob}(\underline{\mathcal P})$.  Let
	 $\epsilon_1,\ldots,\epsilon_k$ be the (finitely many)
	 transition/discontinuity points of $\theta$. For $x,x'\in X$
	 define $d_X(x,x')=\min\{\epsilon_i\}$ s.t. $x,x'$ belong to
	 same block of $\theta(\epsilon_i)$.  
	
	 This is a pseudo metric on $X$. Indeed, pick points $x,x'$
	 and $x''$ in $X$. Let $\epsilon_1$ and $\epsilon_2$ be
	 minimal s.t. $x,x'$ belong to the same block of
	 $\theta(\epsilon_1)$ and $x',x''$ belong to the same block of
	 $\theta(\epsilon_2)$. Let
	 $\epsilon_{12}:=\max(\epsilon_1,\epsilon_2)$. Since
	 $(X,\theta)$ is a persistent set (Definition
	 \ref{def:persistent-sets}), $\theta(\epsilon_{12})$ must have
	 a block $\mathcal B$ s.t. $x,x'$ and $x''$ all lie in
	 $\mathcal B$. Hence $d_X(x,x'')\leq \epsilon_{12}\leq
	 \epsilon_1+\epsilon_2=d_X(x,x')+d_X(x',x'')$.

	 Finally, Kleinberg's \underline{third condition},
	 \emph{consistency}, could be viewed as a rudimentary example
	 of functoriality. His morphisms are similar to the ones in
	 $\underline{\mathcal M}^{gen}$.

\subsection{Functoriality over $\underline{{\mathcal M}}^{mon}$}
In this section, we illustrate how relaxing the functoriality permits
more clustering algorithms.  In other words, we will restrict ourselves
to $\underline{{\mathcal M}}^{mon}$ which is smaller (less stringent) than
$\underline{{\mathcal M}}^{gen}$ but larger (more stringent) than
$\underline{{\mathcal M}}^{iso}$.  We consider the restriction of ${\mathcal
R}^{gen}$ to the category $\underline{{\mathcal M}}^{mon}$.  For any
metric space and every value of the persistence parameter $r$, we will
obtain a partition of the underlying set $X$ of the metric space in
question, and the set of equivalence classes under $\sim _r$.  For any
$x \in X$, let $[x]_r$ be the equivalence class of $x$ under the
equivalence relation $\sim _r$, and define $c(x) = \# [x]_r$.  For any
integer $m$, we now define $X_m \subseteq X$ by $X_m = \{ x \in X|
c(x) \geq m\} $.  We note that for any morphism $f: X \rightarrow Y$
in $\underline{{\mathcal M}}^{mon}$, we find that $f(X_m) \subseteq Y_m$.
This property clearly does not hold for more general morphisms.  For
every $r$, we can now define a new equivalence relation $\sim ^m_r $
on $X$, which refines $\sim _r$, by requiring that each equivalence
class of $\sim _r$ which has cardinality $\geq m$ is an equivalence
class of $\sim _r^m$, and that for any $x $ for which $c(x) < m$, $x$
defines a singleton equivalence class in $\sim ^m_r$. We now obtain a
new persistent set $(X,\theta ^m )$, where $\theta ^m (r)$ will denote
the partition associated to the equivalence relation $\sim ^m_r$.  It
is readily checked that $X \rightarrow (X, \theta ^m )$ is functorial
on $\underline{{\mathcal M}}^{mon}$.  This scheme could be motivated by
the intuitition that one does not regard clusters of small cardinality
as significant, and therefore makes points lying in small clusters
into singletons, where one can then remove them as representing
``outliers".

\section{Metric stability and convergence properties of ${\mathcal R}^{gen}$} \label{sec:stab}
In this section we briefly discuss some further properties of
${\mathcal R}^{gen}$ (single linkage dendrograms). We will provide
quantitative results on the
\emph{stability} and \emph{convergence/consitency} properties of this functor (algorithm). To the best of our knowledge, the only other related  results obtained for this algorithm appear in \cite{hartigan}. The issues of stability and convergence/consistency of clustering algorithms have brought back into attention recently, see \cite{ulrike_general,sober} and references therein. In Theorem \ref{theo:conv} besides proving stability, we prove convergence in a simple setting.

Given finite metric spaces $(X,d_X)$ and $(Y,d_Y)$, our goal is to
define a distance between the persistence objects $\theta^{d_X}$ and
$\theta^{d_Y}$ respectively produced by ${\mathcal R}^{gen}$. We know
that this functor actually outputs dendrograms (rooted trees), which
have a natural metric structure attached to them. Moroever, it is well
known that rooted trees are uniquely characterized by their distance
matrix,
\cite{book-trees}. 

For a finite metric space $(X,d_X)$ consider the derived metric space
$(X,\varepsilon_X)$ with the same underlying set and metric
\beq{def:varepsilon-X}
\varepsilon_X(x,x'):=\min\{\varepsilon\geq 0|\,x\sim_\varepsilon x'\}.
\eeq

Note that $(X,\theta^{d_X})$ can therefore obviously be regarded as
the metric space $(X,\varepsilon_X)$, cf. with the construction of the
metric in section \ref{sec:klein-conds}. We now check that indeed
$\varepsilon_X$ defines a metric on $X$.
\begin{Proposition}
	For any finite metric space $(X,d_X)$, $(X,\varepsilon_X)$ is
	also a metric space.
\end{Proposition}
\begin{proof}
(a) Since $d_X$ is a metric on $X$, it is obvious that
$\varepsilon_X(x,x')=0$ implies $x=x'$.  (b) Symmetry is also obvious
since $\sim_\varepsilon$ is an equivalence relation.  (c) Triangle
inequality: Pick $x,x',x''\in X$.  Let
$\varepsilon_X(x,x')=\varepsilon_1$ and
$\varepsilon_X(x',x'')=\varepsilon_2$. Then, there exist points
$a_0,a_1,\ldots,a_j$ and $b_0,b_1,\ldots,b_k$ in $X$ with $a_0=x$,
$a_j=x'=b_0$, $b_k=x''$ and $d_X(a_i,a_{i+1})\leq \varepsilon_1$ for
$i=0,\ldots,j-1$ and $d_X(b_i,b_{i+1})\leq \varepsilon_2$ for
$i=0,\ldots,k-1$. Consider the points $
\{c_i\}_{i=0}^{j+k+1}=\{a_0,\ldots,a_j,b_1,\ldots,b_k\}.$\\
 Then $d_X(c_i,c_{i+1})\leq \max(\varepsilon_1,\varepsilon_2)\leq
\varepsilon_1+\varepsilon_2$. Hence \mbox{$x\sim_{\varepsilon_{12}} x''$} with $\varepsilon_{12}=\varepsilon_1+\varepsilon_2$ and then by definition  (\ref{def:varepsilon-X}), $\varepsilon_X(x,x'')\leq \varepsilon_X(x,x')+\varepsilon_X(x',x'').$
\end{proof}

In order to compare the outputs of ${\mathcal R}^{gen}$ on two
different finite metric spaces $(X,d_X)$ and $(X',d_{X'})$ we will
instead compare the metric space representations of those outputs,
$(X,\varepsilon_X)$ and $(X',\varepsilon_{X'})$, respectively. For
this purpose, we choose to work with the Gromov-Hausdorff distance
which we define now,
\cite{burago-book}.

\begin{Definition}[Correspondence]
For sets $A$ and $B$, a subset $R\subset A\times B$ is a
\emph{correspondence} (between $A$ and $B$) if and and only if
\begin{itemize}
\item $\forall$ $a\in A$, there exists $b\in B$ s.t. $(a,b)\in R$
\item $\forall$ $b\in B$, there exists $a\in X$ s.t. $(a,b)\in R$
\end{itemize}
\end{Definition}

Let ${\mathcal R}(A,B)$ denote the set of all possible correspondences
between sets $A$ and $B$.

Consider finite metric spaces $(X,d_X)$ and $(Y,d_Y)$. Let
$\Gamma_{X,Y}:X\times Y \bigtimes X\times Y\rightarrow \R^+$ be given
by $$(x,y,x',y')\mapsto |d_X(x,x')-d_Y(y,y')|.$$ Then, the
\textbf{Gromov-Hausdorff} distance between $X$ and $Y$ is given by
\beq{dgh}
	\dgro{X}{Y} := \inf_{R\in{\mathcal
	R}(X,Y)}\,\,\max_{(x,y),(x',y')\in R}\Gamma_{X,Y}(x,y,x',y')
\eeq
\begin{Remark}
This expression defines a \textbf{metric} on the set of (isometry
classes of) finite metric spaces, \cite{burago-book} (Theorem
7.3.30). 
\end{Remark}

One has:
\begin{Proposition}\label{prop:varep}
For any finite metric spaces $(X,d_X)$ and $(Y,d_Y)$
$$\dgro{(X,d_X)}{(Y,d_Y)}\geq
\dgro{(X,\varepsilon_X)}{(Y,\varepsilon_Y)}.$$
\end{Proposition}
\begin{proof}
Let $\eta=\dgro{(X,d_X)}{(Y,d_Y)}$ and $R\in{\mathcal R}(X,Y)$
s.t. $|d_X(x,x')-d_Y(y,y')|\leq \eta$ for all $(x,y),(x',y')\in
R$. Fix $(x,y)$ and $(x',y')\in R$. Let $x_0,\ldots,x_m\in X$ be s.t.
$x_0=x$, $x_m=x'$ and $d_X(x_i,x_{i+1})\leq
\varepsilon(x,x')$ for all $i=0,\ldots,m-1$.
Let $y=y_0,y_1,\ldots,y_{m-1},y_m=y'\in Y$ be s.t. $(x_i,y_i)\in R$
for all $i=0,\ldots,m$ (this is possible by definition of $R$). Then,
$d_Y(y_i,y_{i+1})\leq d_X(x_i,x_{i+1})+\eta\leq
\varepsilon_X(x,x')+\eta$ for all $i=0,\ldots,m-1$
and hence $\varepsilon_Y(y,y')\leq
\varepsilon_X(x,x')+\eta$. By exchanging the roles of $X$ and $Y$ one obtains the inequality $\varepsilon_X(x,x')\leq
\varepsilon_Y(y,y')+\eta$. This means $|\varepsilon_X(x,x')-\varepsilon_Y(y,y')|\leq \eta$. Since $(x,y),(x',y')\in R$ are arbitrary, and upon recalling the definition of the Gromov-Hausdorff distance  we obtain the desired conclusion.
\end{proof}

Proposition \ref{prop:varep} will allow us to quantify
\emph{stability} and
\emph{convergence}. We provide deterministic arguments. The same
construction, essentially, yields similar results under the assumption
that $(Z,d_Z)$ is enriched with a (Borel) probability measure and one
takes i.i.d. samples w.r.t. this probability measure. Assume $(Z,d_Z)$
is an underlying (perhaps ``continuous'') metric space from which
different finite samples are drawn. We would like to see,
quantitatively, (1) how the results yielded by ${\mathcal R}^{gen}$
differ when applied to those different sample sets (which possibly
contain different numbers of points), this is \emph{stability} and (2)
that when the underlying metric space is partitioned, there is
\emph{convergence} and \emph{consistency} in a precise sense.

Assume $A$ is a finite set and let $W:A\times A\rightarrow \R^+$ be a
symmetric map. Using the usual path-length construction, we endow $A$
with the (pseudo)metric
$$d_A(a,a'):=\min\sum_{k=0}^{m-1}W(a_{k},a_{k+1})$$ where the minimum
is taken over $m$ and all sets of $m+1$ points $a_0,\ldots,a_m$ such
that $a_0=a$ and $a_m=a'$. We denote $d_A = {\mathcal L}(W)$. This is
a standard construction, see \cite{bridson} \S1.24.

For a compact metric space $(Z,d_Z)$ and any two of its compact
subsets $Z_1,Z_2$ let $$D_Z(Z_1,Z_2)=\min_{z_1\in Z_1}\min_{z_2\in
Z_2}d_Z(z_1,z_2).$$

For $(Z,d_Z)$ compact and $X,X'\subset Z$ compact, let $d_{\mathcal
H}^Z(X,X')$ denote the Hausdorff distance (in $Z$) between $X$ and
$X'$,
\cite{burago-book}. For any $X\subset Z$ let $R(X):=d_{\mathcal
H}^Z(X,Z)$. Intuitively this number measures how well $X$ approximates
$Z$. One says that $X$ is an $R(X)$-\emph{covering} of $Z$ or an
$R(X)$-net of $Z$.

The following theorem summarizes our main results regarding metric
stability and convergence/consistency. The situation described by the theorem is
depicted in Figure \ref{fig:theoconv}.

\begin{Theorem}\label{theo:conv}
Assume $(Z,d_Z)$ is a compact metric space. Let $X$ and $X'$ be any
two finite sets of points sampled from $Z$.  Endow these two sets with
the (restricted) metric $d_Z$. Then,

\begin{enumerate}
 	
	\item (Finite Stability) $\dgro{(X,\varepsilon_X)}{(X',\varepsilon_{X'})}\leq
	2(R(X)+R(X')).$

	\item (Asymptotic Stability) As $\max\left(R(X),R(X')\right)\rightarrow 0$ one has
	$\dgro{(X,\varepsilon_X)}{(X',\varepsilon_{X'})}\rightarrow
	0.$

	\item (Convergence/consistency) Assume in addition that
	$Z=\cup_{\alpha\in A}Z_\alpha$ where $A$ is a finite index set
	and $Z_\alpha$ are compact, disjoint and path-connected
	sets. Let $(A,d_A)$ be the finite metric space with underlying
	set $A$ and metric given by $d_A:={\mathcal L}(W)$ where
	$W(\alpha,\alpha'):=D_Z(Z_{\alpha},Z_{\alpha'})$ for
	$\alpha,\alpha'\in A$.\footnote{Since the $Z_\alpha$ are
	disjoint, $d_A$ is a true metric on $A$.}  Then, as
	$R(X)\rightarrow 0$ one has
	$$\dgro{(X,\varepsilon_X)}{(A,\varepsilon_A)}\rightarrow 0.$$

\end{enumerate}
\end{Theorem}

\begin{figure}[htb]
	\centering
	\includegraphics[width=1\linewidth]{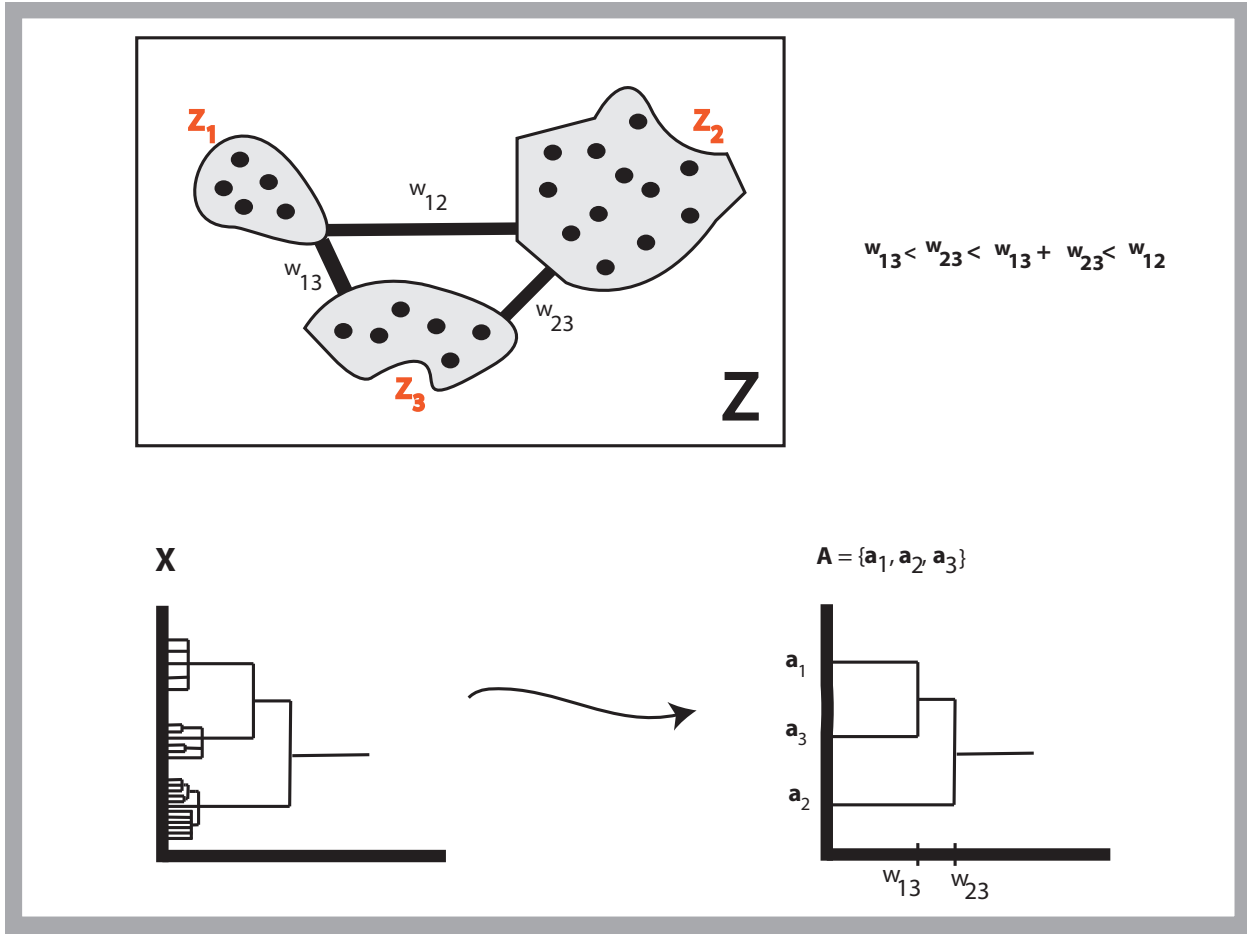}
	\caption{Explanation of Theorem \ref{theo:conv}. \emph{Top}: A
	space $Z$ composed of $3$ disjoint path connected parts, $Z_1,
	Z_2$ and $Z_3$. The black dots are the points in the finite
	sample $X$. In the figure, $w_{ij} = D_Z(Z_i,Z_j)$, $1\leq
	i\neq j\leq 3$. \emph{Bottom Left}: The dendrogram
	representation of $(X,\theta^{d_X})$. \emph{Bottom Right} The
	dendrogram representation of the persistent set
	$(A,\theta^{d_A})$. Note that $d_A(a_1,a_2)=w_{13}+w_{23}$,
	$d_A(a_1,a_3)=w_{13}$ and $d_A(a_2,a_3)=w_{23}$. As
	$R(X)\rightarrow 0$, $(X,\theta^{d_X})\rightarrow
	(A,\theta^{d_A})$ in the Gromov-Hausdorff sense, see text for
	details.} \label{fig:theoconv}
\end{figure}
\begin{proof}
Let $\delta>0$ be s.t.  $\min_{\alpha\neq
\beta}D_Z(Z_\alpha,Z_\beta)\geq \delta.$

Claim 1. follows from Proposition \ref{prop:varep}: let $d_X$
(resp. $d_{X'}$) equal the restriction of $d_Z$ to $X\times X$
(resp. $X'\times X'$). Then, by the triangle inequality for the
Gromov-Hausdorff distance $$\dgro{X}{Z} +
\dgro{X'}{Z}\geq
\dgro{(X,\varepsilon_X))}{(X',\varepsilon_{X'}))}.$$
Now, the claim follows from the fact that whenever $Z\subset Z'$,
$\dgro{Z'}{Z}\leq 2d_{\mathcal H}^{Z}(Z,Z')=2R(Z')$, \cite{burago-book},
\S7.3.

Claim 2. follows directly from claim 1.  

We now prove the third claim. For each $x\in X$ let $\alpha(x)$ denote
the index of the path connected component of $Z$ s.t. $x\in
Z_{\alpha(x)}$.  Assume, $R(X)<
\frac{\delta}{2}$. Then, it is clear that $\#\left(Z_\alpha\cap X\right)\geq 1$ for all $\alpha\in A$.  
Then it follows that $R=\{(x,\alpha(x))|x\in X\}$ belongs to
${\mathcal R}(X,A)$.  We prove below that for all
$x,x'\in X$ $$\varepsilon_A(\alpha(x),\alpha(x')) \stackrel{(1)}{\leq}
\varepsilon_X(x,x') \stackrel{(2)}{\leq}
\varepsilon_A(\alpha(x),\alpha(x')) + 2R(X).$$

It follows immediately from the definition of $W$ that for all
$y,y'\in X$, $W(\alpha(y),\alpha(y'))\leq d_X(y,y')$. From the
definition of $d_A$ it follows that $W(\alpha,\alpha')\geq
d_A(\alpha,\alpha')$. Then in order to prove (1) pick $x_0,\ldots,x_m$
in $X$ with $x_0=x$, $x_m=x'$ and $d_X(x_i,x_{i+1})\leq
\varepsilon_X(x,x')$. Consider the points in $A$ given by
$\alpha(x)=\alpha(x_0),\ldots, \alpha(x_m)=\alpha(x')$. Then,
$d_A(\alpha(x_i),\alpha(x_{i+1}))\leq W(\alpha(x_i),\alpha(x_{i+1}))\leq d_X(x_i,x_{i+1})\leq
\varepsilon_X(x,x')$ for $i=0,\ldots,m-1$ by the claim above. Hence
(1) follows.

We now prove (2). Assume first that $\alpha(x)=\alpha(x')=\alpha$.
Fix $\epsilon_0>0$ small. Let $\gamma:[0,1]\rightarrow Z_\alpha$ be a
continuous path s.t. $\gamma(0)=x$ and $\gamma(1)=x'$. Let
$z_1,\ldots,z_m$ be points on $\mbox{image}(\gamma)$ s.t. $z_0 = x$,
$z_m = x'$ and $d_X(z_i,z_{i+1}) \leq {\epsilon_0}$,
$i=0,\ldots,m-1$. By hypothesis, one can find
$x=x_0,x_1,\ldots,x_{m-1},x_m=x'$ s.t. $d_Z(x_i,z_i)\leq R(X)$. Hence
$d_X(x_i,x_{i+1})\leq {\epsilon_0}+2R(X)$ and hence
$\varepsilon_X(x,x')\leq {\epsilon_0}+2R(X)$. Let
$\epsilon_0\rightarrow 0$ to obtain the desired result.

Now if $\alpha=\alpha(x)\neq \alpha(x')=\beta$, let
$\alpha_0,\alpha_1,\ldots,\alpha_l\in A$ be s.t. $\alpha_0=\alpha(x)$,
$\alpha_l=\alpha(x')$ and $d_A(\alpha_j,\alpha_{j+1})\leq
\varepsilon_A(\alpha(x),\alpha(x'))$ for $j=0,\ldots,l-1$.

By definition of $d_A$, for each $j=0,\ldots,l-1$ one can find a path
${\mathcal C}_j=\{\alpha_j^{(0)}, \ldots, \alpha_j^{(r_j)}\}$
s.t. $\alpha_j^{(0)}=\alpha_j$, $\alpha_{j}^{(r_j)}=\alpha_{j+1}$ and
$\sum_{i=0}^{r_j-1}W(\alpha_{j}^{(i)},\alpha_{j}^{(i+1)})=d_A(\alpha_j,\alpha_{j+1})\leq
\varepsilon_A(\alpha,\beta).$ It follows that
$W(\alpha_{j}^{(i)},\alpha_{j}^{(i+1)})\leq
\varepsilon_A(\alpha,\beta)$ for $i=0,\ldots,r_{j}-1$. Consider the path 
${\mathcal C}=\{\widehat{\alpha}_0,\ldots, \widehat{\alpha}_s\}$ in
$A$ joining $\alpha$ to $\beta$ given by the concatenation of all the
${\mathcal C}_j$. By eliminating repeated consecutive elements in
${\mathcal C}$ if necessary, one can assume that
$\widehat{\alpha}_i\neq \widehat{\alpha}_{i+1}$. By construction
$W(\widehat{\alpha}_i,\widehat{\alpha}_{i+1})\leq
\varepsilon_A(\alpha,\beta)$ and $\widehat{\alpha}_0=\alpha$, $\widehat{\alpha}_s=\beta$.
We will now lift $\mathcal C$ into a path in $Z$ joining $x$ to $x'$.

Note that by compactness, for all $\nu,\mu\in A$, $\nu\neq
\mu$ there exist $z_{\nu,\mu}^\nu\in Z_\nu$ and $z_{\nu,\mu}^\mu\in Z_\mu$ s.t. 
$W(\nu,\mu)=d_Z(z_{\nu,\mu}^\nu,z_{\nu,\mu}^\mu)$. Consider the path
$\mathcal G$ in $Z$ given by $${\mathcal
G}=\{x,z_{\widehat{\alpha}_0,\widehat{\alpha}_1}^{\widehat{\alpha}_0},z_{\widehat{\alpha}_0,\widehat{\alpha}_1}^{\widehat{\alpha}_1},\ldots,z_{\widehat{\alpha}_{s-1},\widehat{\alpha}_{s}}^{\widehat{\alpha}_s},x'\}.$$
For each point $g\in \mathcal G$ pick a point $x(g)\in X$
s.t. $d_Z(g,x(g))\leq R(X)$. This is possible by definition of
$R(X)$. Let ${\mathcal G}' = \{x_0,x_1,\ldots,x_t\}$ be the resulting
path in $X$.  Notice that if $\alpha(x_t)\neq\alpha(x_{t+1})$ then
$d_X(x_t,x_{t+1})\leq 2R(X)+W(\alpha(x_t),\alpha(X_{t+1}))$ by
the triangle inequality. Also, by construction,$(*)\hspace{.5in}W(\alpha(x_t),\alpha(x_{t+1}))\leq
\varepsilon_A(\alpha,\beta).$

 Now, we claim that $$\varepsilon_X(x,x')\leq
\max_{t}W(\alpha(x_t),\alpha(x_{t+1}))+2R(X).$$
This claim will follow from the simple observation that
$\varepsilon_X(x,x')\leq \max_t \varepsilon_X(x_t,x_{t+1}).$ If
$\alpha(x_t)=\alpha(x_{t+1})$ we already proved that
$\varepsilon_X(x_t,x_{t+1})\leq 2R(X)$. If on the other hand
$\alpha(x_t)\neq \alpha(x_{t+1})$ then,
$\varepsilon_X(x_t,x_{t+1})\leq 2R(X)+W(\alpha(x_t),\alpha(x_{t+1}))$
and hence the claim. Combine this fact with $(*)$ to conclude the
proof of (2). Putting (1) and (2) together we have
$$\dgro{(X,\varepsilon_X)}{(A,\varepsilon_A)}\leq 2R(X)$$ and the
conclusion follows by letting $R(X)\rightarrow 0$.

\end{proof}

\section{Functoriality and bootstrap  clustering}\label{sec:zigzags}\label{bootstrap}
In the previous section, we have observed that by encoding the output
of a clustering scheme as diagram (i.e. as a persistent set or
dendrogram) allows one to assess stability of the clustering obtained
from the scheme.  In this section, we will demonstrate that another
use of functoriality can be used to assess stability of clustering
schemes whose output is simply a partition of the underlying point
cloud.  We begin by recalling the basics of the {\em bootstrap} method
developed by B. Efron \cite{efron}.  The bootstrap considers a set of
point cloud data $\Bbb{X}$, and repeatedly samples (with replacement)
collections of (say) $n$ elements from $\Bbb{X}$.  For each sample,
one measures of central tendency such as means, medians, variances,
are computed, and the distribution of these measures as a statistic
are studied.  It is understood that such computations are more
informative than the measures computed a single time on the full set
$\Bbb{X}$.  We wish to perform a similar analysis for clustering.  The
difficulty is that the output of clustering is not a single numerical
statistic, but is rather a structural, qualitative output. We will now
show how functoriality can be used to assess compatibilty of
clusterings of subsamples, and thereby obtain a method for confirming
that clustering is a significant feature of the data rather than an
artifact.

In the context of clustering these bootstrapping ideas arise when
dealing with massive datasets: one is forced to analysing several
smaller, more manageable random subsamples of the original data to
produce partial pictures of the underlying clustering structure. The
problem then is how to agglomerate all this information together.

In this section, for us, a clustering scheme will denote any rule $C$
which assigns to every finite metric space $S$ a partition ${\mathcal
P}_{C}(S)$. We write ${\mathcal B}_C(S)$ for the set of blocks of the
partition ${\mathcal P}_C(S)$.  If we are given two finite metric spaces
$S,T$, an \underline{embedding} from $S$ to $T$, is an injective set
map $\iota: S
\hookrightarrow T$, so that $d_T(\iota (x), \iota (x^{\prime})) =
d_S(x,x^{\prime})$.  Given any partition ${\mathcal P}$ of a metric space
$T$, and given any set map $\varphi : S \rightarrow T$, we write
$\varphi ^*({\mathcal P})$ for the partition of $S$ which places
$s,s^{\prime} \in S$ in the same block if and only if $\varphi (s) $
and $\varphi (s^{\prime}) $ lie in the same block of ${\mathcal P}$.  The
clustering scheme $C$ is now said to be {\bf I-functorial } if ${\mathcal
P}_C(S) $ refines $\iota ^* ({\mathcal P}_C(T))$ for any embedding $\iota:
S \rightarrow T$.  Note that for any I-functorial clustering scheme
$C$, there is an induced map ${\mathcal B}_C(\iota): {\mathcal B}_C(S)
\rightarrow {\mathcal B}_C(T)$ for any embedding $\iota : S
\hookrightarrow T$.  An example of an I-functorial clustering scheme is single linkage clustering for a fixed threshhold $\epsilon$.  


Now let $\Bbb{X}$ be a set of point cloud data, equipped with a metric
$d$.  We build collections of samples $S_i \subseteq \Bbb{X}$ of size
$n $ from $\Bbb{X}$, { with replacement}, for $1 \leq i \leq N$.  We
assume we are given an I-functorial clustering scheme $C$.  We note
that each of the samples $S_i$ and the sets $S_i \cup S_{i+1}$ are
finite metric spaces in their own right, that the natural inclusions
$S_i \hookrightarrow S_i \cup S_{i+1}$ and $S_{i+1} \hookrightarrow
S_i \cup S_{i+1}$ are embeddings of finite metric spaces.  It follows
from the I-functoriality of the clustering scheme $C$ that we obtain a
diagram of sets in Figure \ref{diagram-g}

	\begin{figure}[htb] \centering
	\includegraphics[width=1\linewidth]{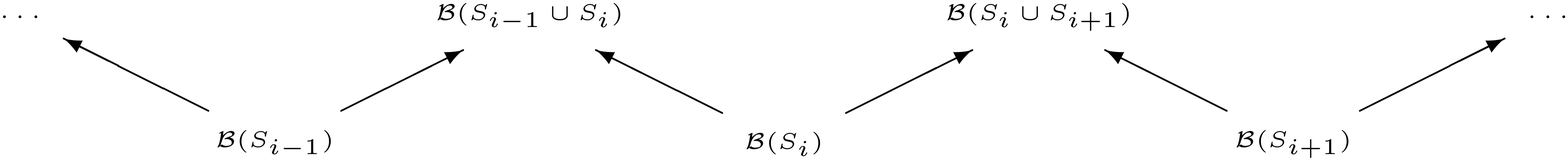}
	\caption{Diagram of sets obtained via I-functoriality of the
	clustering scheme.}  \label{diagram-g} \end{figure}


We will refer to such a diagram as a {\em zig zag}.  In an intuitive
sense, this diagram now carries information about the stability or the
significance of the clustering.  The informal idea is that sequences
of the form $\{ x_{\nu} \}_{\nu = s}^{t}$ (formed by consecutive
elements), with $x_{\nu} \in {\mathcal B}_C(S_{\nu})$, and with ${\mathcal
B}_C(i^+_{\nu})(x_{\nu}) ={\mathcal B}_C(i^-_{\nu +1})(x_{\nu + 1})$
should describe small scale pictures of a clustering of $\Bbb{X}$,
where $i^+_{\nu}: S_{\nu}
\hookrightarrow S_{\nu} \cup S_{\nu +1 }$ and $i^-_{\nu +1 }: S_{\nu
+1 } \hookrightarrow S_{\nu} \cup S_{\nu +1 }$ are the inclusions.
Informally, the idea is that ``compatible families" of clusterings of
the samples $S_{\nu}$ should correspond to clusterings of the entire
set $\Bbb{X}$.  Of course, the length of the sequence ($t-s+1$) must
be significant.  A single pair of compatible clusters will not be as
significant as a long sequence.  The problem with this idea as stated
is that it is very hard to make precise the definitions of the
sequences, and to describe them.

Unlike the case of ordinary persistent sets, where dendrograms provide
a straightforward visualization of all such structure, we believe that
in the case of zig-zags of sets no such simple representation is
possible. However, there turns out to be (see below) a readily
computable analogue of the \emph{persistence barcode},
\cite{ghrist,persistence}. We now see how this works.

We note first that this situation has certain things in common with
dendrograms.  Rooted trees can be viewed as diagrams of sets of the
form $$ X_0 \stackrel{f_0}{\rightarrow} X_1
\stackrel{f_1}{\rightarrow} \cdots X_{n-1}
\stackrel{f_{n-1}}{\rightarrow} X_{n} \rightarrow \cdots$$ for which
there is an integer $N$ so that $X_k$ consists of one element for all
$k \geq N$.  The smallest such $N$ will be called the depth of the
tree, $d$.  One constructs a tree from such a diagram by forming the
disjoint union $\coprod_{i=0}^{d} X_i \times [0,1]$, and then forms
the quotient by the equivalence relation generated by the equivalences
$x \times 1 \simeq f_s(x) \times 0$ for all $x \in X_s$ and $s < d$.
The set $X_l$ will now correspond to the nodes of depth $d-l$ in the
tree.  The tree representation turns out to be a useful representation
of structure of the sets of clusters as a set varying with a
threshhold parameter.  Given instead a zig zag diagram as above, it is
again possible to construct a graph which represents the data, but it
is harder to make useful sense of it, since it is a fairly general
graph.  Nonetheless, it turns out that it is possible to obtain a
useful partial description using algebraic techniques.

One begins with a field $k$ (typically $\Bbb{F}_2$, the field with two
elements), and constructs for each of the sets ${\mathcal B}(S_i )$ and
${\mathcal B}(S_{i} \cup S_{i +1})$ in the zig zag the corresponding
vector spaces $k[{\mathcal B}(S_i )]$ and $k[{\mathcal B}(S_{i} \cup S_{i
+1})]$, i.e. vector spaces with the given sets as bases.  The zig zag
diagram now gives rise to a diagram of vector spaces and linear
transformations of the same shape.  It turns out that there is an
algebraic classification of such diagrams up to isomorphism.  To
describe this classification, we will describe every zig zag diagram
as a family of vector spaces $\{ V_i \}_{i }$, equipped with linear
transformations $\lambda _i : V_{2i} \rightarrow V_{2i + 1}$ and $\mu
_i : V_{2i} \rightarrow V_{2i -1 }$. Given integers $a \leq b$, we
denote by $Z[a,b]$ the zig zag diagram for which $V_i = k$ for all $a
\leq i \leq b$, and $V_i = \{ 0 \} $ for $i \notin [a,b]$, and for
which every possible non-zero linear transformation is equal to the
identity.  For example, $Z[3,6]$ is the diagram $$ \cdots \{ 0 \}
\rightarrow V_3 \stackrel{id}{\leftarrow} V_4
\stackrel{id}{\rightarrow} V_5 \stackrel{id}{\leftarrow} V_6
\rightarrow \{ 0 \} \cdots $$ where $V_3,V_4, V_5, V_6 = k$.  Note
that these diagrams are parametrized by closed intervals with integer
endpoints.

We now have the  following theorem of Gabriel (see \cite{gabriel}).
\begin{Theorem} Every zig zag diagram is isomorphic to a direct sum of diagrams of the form $[a_i, b_ i]$, and the decomposition is unique up to reordering of the summands.  
\end{Theorem} 

\section{Discussion} \label{various} 

We have presented the ideas of functoriality and persisence as useful
organizing principles for clustering algorithms.  We have made
particular choices of category structures on the collection of finite
metric spaces, as well as for the notion of multiscale/resolution
sets.  One can imagine different notions of morphisms of metric spaces
and of persistent sets.  For example, the idea of {\em
multidimensional persistence} (see \cite{multid}) could provide
methods which in addition to the parameter $r$ could track density as
estimated by some estimator, giving a more informative picture of the
dataset.  It also appears likely that from the point of view described
here, it will in many cases be possible, given a collection of
constraints on a clustering functor, to determine the {\em universal }
one satisfying the constraints.  One could therefore use sets of
constraints as the definition of clustering functors.

We believe that the conceptual framework presented here can be a
useful tool in reasoning about clustering algorithms.  We have also
shown that clustering methods which have some degree of functoriality
admit the possibility of certain kind of qualitative geometric
analysis of datasets which can be quite valuable.  The general idea
that the morphisms between mathematical objects (together with the
notion of functoriality) are critical in many situations is
well-established in many areas of mathematics, and we would argue that
it is valuable in this statistical situation as well.

We have also discussed how to obtain quantitative stability,
consistency and convergence results using a metric space
representation of the output of clustering algorithms. We believe
these tools can also contribute to the understanding of theoretical
questions about clustering as well.

Finally we would like to comment on the fact that functoriality ideas
and metric based study complement eachother. In the sense that using
functoriality, first, one can reason about global stability or
rigidity of methods in order to identify a class of them that is
sensible, and then, by applying metric tools one can understand the
behaviour/convergence as, say, the number of samples goes to infinity,
or to the quantify error in approximating the underlying reality when
only finitely many samples are used.

\bibliography{clustering_biblio}

\begin{thebibliography}{BDvLP06}

\bibitem[BBI01]{burago-book}
D.~Burago, Y.~Burago, and S.~Ivanov.
\newblock {\em A Course in Metric Geometry}, volume~33 of {\em AMS Graduate
  Studies in Math.}
\newblock American Mathematical Society, 2001.

\bibitem[BDvLP06]{sober}
Shai Ben-David, Ulrike von Luxburg, and D{\'a}vid P{\'a}l.
\newblock A sober look at clustering stability.
\newblock In G{\'a}bor Lugosi and Hans-Ulrich Simon, editors, {\em COLT},
  volume 4005 of {\em Lecture Notes in Computer Science}, pages 5--19.
  Springer, 2006.

\bibitem[BH99]{bridson}
Martin~R. Bridson and Andr{\'e} Haefliger.
\newblock {\em Metric spaces of non-positive curvature}, volume 319 of {\em
  Grundlehren der Mathematischen Wissenschaften [Fundamental Principles of
  Mathematical Sciences]}.
\newblock Springer-Verlag, Berlin, 1999.

\bibitem[BK72]{bk}
A.~K. Bousfield and D.~M. Kan.
\newblock {\em Homotopy limits, completions and localizations}.
\newblock Springer-Verlag, Berlin, 1972.
\newblock Lecture Notes in Mathematics, Vol. 304.

\bibitem[CIdSZ08]{mumford}
G.~Carlsson, T.~Ishkhanov, V.~de~Silva, and A.~Zomorodian.
\newblock On the local behavior of spaces of natural images.
\newblock {\em IJCV}, 76(1):1--12, January 2008.

\bibitem[CSZon]{gsan}
Gunnar Carlsson, Gurjeet Singh, and Afra Zomorodian.
\newblock Gr\"obner bases and multidimensional persistence.
\newblock Technical report, Department of Mathematics, Stanford University., in
  preparation.

\bibitem[CZ07]{multid}
Gunnar Carlsson and Afra Zomorodian.
\newblock The theory of multidimensional persistence.
\newblock In {\em SCG '07: Proceedings of the twenty-third annual symposium on
  Computational geometry}, pages 184--193, New York, NY, USA, 2007. ACM.

\bibitem[Efr79]{efron}
B.~Efron.
\newblock Bootstrap methods: another look at the jackknife.
\newblock {\em Ann. Statist.}, 7(1):1--26, 1979.

\bibitem[Ghr08]{ghrist}
Robert Ghrist.
\newblock Barcodes: The persistent topology of data.
\newblock {\em Bull. Amer. Math. Soc.}, 45(2):61--75, 2008.

\bibitem[GR97]{gabriel}
P.~Gabriel and A.~V. Roiter.
\newblock {\em Representations of finite-dimensional algebras}.
\newblock Springer-Verlag, Berlin, 1997.
\newblock Translated from the Russian, With a chapter by B. Keller, Reprint of
  the 1992 English translation.

\bibitem[Har81]{hartigan}
J.~A. Hartigan.
\newblock Consistency of single linkage for high-density clusters.
\newblock {\em J. Amer. Statist. Assoc.}, 76(374):388--394, 1981.

\bibitem[Isb64]{isbell}
J.~R. Isbell.
\newblock Six theorems about injective metric spaces.
\newblock {\em Comment. Math. Helv.}, 39:65--76, 1964.

\bibitem[JD88]{clusteringref}
Anil~K. Jain and Richard~C. Dubes.
\newblock {\em Algorithms for clustering data}.
\newblock Prentice Hall Advanced Reference Series. Prentice Hall Inc.,
  Englewood Cliffs, NJ, 1988.

\bibitem[JS71]{math-taxo}
Nicholas Jardine and Robin Sibson.
\newblock {\em Mathematical taxonomy}.
\newblock John Wiley \& Sons Ltd., London, 1971.
\newblock Wiley Series in Probability and Mathematical Statistics.

\bibitem[Kle02]{kleinberg}
Jon~M. Kleinberg.
\newblock An impossibility theorem for clustering.
\newblock In Suzanna Becker, Sebastian Thrun, and Klaus Obermayer, editors,
  {\em NIPS}, pages 446--453. MIT Press, 2002.

\bibitem[McC02]{mccullaugh}
Peter McCullagh.
\newblock What is a statistical model?
\newblock {\em Ann. Statist.}, 30(5):1225--1310, 2002.
\newblock With comments and a rejoinder by the author.

\bibitem[ML98]{maclane}
Saunders Mac~Lane.
\newblock {\em Categories for the working mathematician}, volume~5 of {\em
  Graduate Texts in Mathematics}.
\newblock Springer-Verlag, New York, second edition, 1998.

\bibitem[Mun75]{munkres}
James~R. Munkres.
\newblock {\em Topology: a first course}.
\newblock Prentice-Hall Inc., Englewood Cliffs, N.J., 1975.

\bibitem[Rag82]{raghavan}
Vijay~V. Raghavan.
\newblock Approaches for measuring the stability of clustering methods.
\newblock {\em SIGIR Forum}, 17(1):6--20, 1982.

\bibitem[SMC07]{mapper}
Gurjeet Singh, Facundo M\'emoli, and Gunnar Carlsson.
\newblock {Topological Methods for the Analysis of High Dimensional Data Sets
  and 3D Object Recognition}.
\newblock pages 91--100, Prague, Czech Republic, 2007. Eurographics
  Association.

\bibitem[SS03]{book-trees}
Charles Semple and Mike Steel.
\newblock {\em Phylogenetics}, volume~24 of {\em Oxford Lecture Series in
  Mathematics and its Applications}.
\newblock Oxford University Press, Oxford, 2003.

\bibitem[vL07]{ulrike_spectral}
U.~von Luxburg.
\newblock A tutorial on spectral clustering.
\newblock {\em Statistics and Computing}, 17(4):395--416, 12 2007.

\bibitem[vLBD05]{ulrike_general}
U.~von Luxburg and S.~Ben-David.
\newblock Towards a statistical theory of clustering. presented at the pascal
  workshop on clustering, london.
\newblock Technical report, Presented at the PASCAL workshop on clustering,
  London, 2005.

\bibitem[ZC04]{persistence}
Afra Zomorodian and Gunnar Carlsson.
\newblock Computing persistent homology.
\newblock In {\em SCG '04: Proceedings of the twentieth annual symposium on
  Computational geometry}, pages 347--356, New York, NY, USA, 2004. ACM.

\end{thebibliography}
\bibliographystyle{alpha}
\end{document}